 \documentclass[letterpaper]{article} 

\usepackage{aaai23-r2hcai}  
\usepackage{times}
\usepackage{helvet}
\usepackage{courier}
\usepackage[hyphens]{url}
\usepackage{graphicx}
\usepackage{tikz}
\usepackage{amsmath, amssymb, amsthm}
\usetikzlibrary{arrows.meta,positioning,calc}
\usepackage{pgfplots}
\usetikzlibrary{intersections}
\usepgfplotslibrary{fillbetween} 
\pgfplotsset{compat=1.18}

\newtheorem{theorem}{Theorem}
\newtheorem{definition}{Definition}
\newtheorem{lemma}{Lemma} 
\newtheorem{corollary}{Corollary} 
\newtheorem{assumption}{Assumption} 
\newcommand{\TV}{\mathrm{TV}} 
\theoremstyle{remark} 
\newtheorem{remark}{Remark}

\usepackage[numbers]{natbib} 
\usepackage{caption}
\usepackage{algorithm}
\usepackage{algorithmic}

\usepackage{newfloat}
\usepackage{listings}
\DeclareCaptionStyle{ruled}{labelfont=normalfont,labelsep=colon,strut=off} 
\floatstyle{ruled}
\newfloat{listing}{tb}{lst}{}
\floatname{listing}{Listing}
\title{Intrinsic Stability Limits of Autoregressive Reasoning: Structural Consequences for Long-Horizon Execution}
\author{
    Hsien-Jyh Liao$^1$\equalcontrib \quad
}
\affiliations {
    hjliao123@gmail.com
   
}

\makeatletter
\renewcommand{\subsubsection}{\@startsection{subsubsection}{3}{\z@}%
    {-1.25ex \@plus -1ex \@minus -.2ex}
    {0.5ex \@plus .2ex}
    {\normalfont\normalsize\bfseries}} 

\makeatother

\begin{document}
\maketitle

\begin{abstract}
Large language models (LLMs) demonstrate remarkable reasoning capabilities, yet their performance often deteriorates sharply in long-horizon tasks, exhibiting systematic breakdown beyond certain scales. Conventional explanations primarily attribute this phenomenon to task complexity, such as combinatorial search explosion or long-term credit assignment challenges. In this work, we argue that these explanations are incomplete: even in linear, unbranched tasks without semantic ambiguity, autoregressive execution is subject to an intrinsic stability limit.

We propose that the fundamental constraint on long-horizon reasoning arises from process-level instability in autoregressive generation rather than solely from search or task complexity, reframing long-horizon reasoning as a problem of structural governance. We derive Theorem~A, showing that decision advantage in single-path autoregressive reasoning decays exponentially with execution length, imposing a fundamental bound on maintainable reasoning chains. This result implies a structural consequence: stable long-horizon reasoning requires discrete segmentation, naturally inducing graph-like execution structures such as directed acyclic graphs (DAGs).

Empirical studies in both synthetic environments and real TextWorld tasks reveal observable performance cliffs consistent with theoretical predictions. Our findings provide a dynamical perspective on long-horizon reasoning failure and suggest new limitations on maintaining long-term coherence under purely autoregressive architectures. Furthermore, we highlight that short-horizon evaluation protocols may obscure structural instability, indicating a potential shift from scaling toward structured governance in future reasoning systems.
\end{abstract}

\smallskip
\noindent \textbf{Keywords:}
Autoregressive Reasoning
Long-Horizon Stability
Chain-of-Thought Reasoning
Information-Theoretic Analysis
Structured Reasoning
Inference Dynamics

\section{Introduction}

Large language models (LLMs) have demonstrated remarkable capabilities across a wide range of reasoning tasks. For example, advanced models achieve strong performance on mathematical reasoning benchmarks when guided by structured prompting strategies such as chain-of-thought prompting~\cite{wei2022chain}, tree-of-thought reasoning~\cite{yao2023tree}, graph-based reasoning frameworks~\cite{besta2024graph}, and agent-style paradigms including ReAct~\cite{yao2023react} and Reflexion~\cite{shinn2024reflexion}. These approaches are often interpreted as instances of deliberate or ``System~2'' reasoning~\cite{kahneman2011thinking}, where intermediate reasoning steps are explicitly generated to enhance deliberation and planning.

Despite these advances, an important limitation remains insufficiently formalized. As reasoning sequences grow longer, the stability of reasoning trajectories frequently degrades, even when the underlying task structure does not significantly increase in complexity. Model outputs may exhibit logical drift, internal inconsistency, or contradictions with earlier intermediate conclusions. This observation suggests that the central challenge of generative reasoning is not merely whether reasoning can be completed, but whether it can be stably sustained over extended horizons.

Existing explanations typically attribute long-horizon reasoning failures to task-level complexity, such as combinatorial search explosion, long-term credit assignment difficulties, or accumulating semantic ambiguity. These perspectives align with classical analyses in planning and reinforcement learning, where long-horizon challenges are framed as consequences of large state or policy spaces~\cite{kaelbling1998planning,sutton1999between}. Within language model research, many recent approaches attempt to mitigate such issues through structured reasoning strategies or external scaffolding mechanisms~\cite{wei2022chain,yao2023tree,besta2024graph}. However, these explanations implicitly assume that once task structure is properly managed, the reasoning process itself can be extended arbitrarily.

For autoregressive generative systems, this assumption has not been rigorously examined. If reasoning is viewed as a dynamical process that evolves over time while accumulating internal uncertainty, an alternative hypothesis emerges: long-horizon reasoning failure may arise from intrinsic instability in the execution process itself rather than solely from task complexity. Prior work on neural text generation has identified degeneration phenomena related to sequence continuation~\cite{holtzman2019degeneration}, yet a formal account of stability limits in extended reasoning trajectories remains largely unexplored.

While scaling laws capture aggregate performance trends~\cite{kaplan2020scaling}, they largely remain silent on the structural stability properties of extended reasoning execution. We therefore propose a complementary perspective: long-horizon reasoning should be analyzed as a stability problem of sequential execution dynamics rather than purely as a scaling or search problem.

We formalize this perspective through an information-theoretic and dynamical analysis of autoregressive reasoning. The central theoretical result (Theorem~A) shows that, under mild assumptions, the decision advantage of autoregressive reasoning trajectories decays exponentially with execution length. This demonstrates that instability can emerge as an inherent property of autoregressive execution dynamics rather than merely an artifact of specific architectures. Consequently, stable long-horizon reasoning requires structural segmentation and intermediate stabilization mechanisms, motivating transitions from purely linear execution toward structured graph-like organization.

\section{Related Work}
\label{sec:related_work}

This chapter reviews prior research related to long-horizon reasoning, structured reasoning, and dialogue-oriented systems, positioning our work within the broader research landscape. We organize related studies into four categories: (1) long-horizon decision making and exponential difficulty, (2) chain-based reasoning methods in large language models, (3) structured and graph-based reasoning frameworks, and (4) computational argumentation and dialogue systems. 

Our perspective does not seek to replace existing approaches, but instead addresses a level that has received comparatively limited formal treatment: the dynamical stability of the generative reasoning process itself.

\subsection{Long-Horizon Decision Making and Exponential Difficulty}

In reinforcement learning (RL) and partially observable Markov decision processes (POMDPs), the difficulty of long-horizon tasks has been extensively studied. Classical results demonstrate that, in the absence of informative intermediate signals or identifiable states, the sample complexity required for learning or planning can grow exponentially with decision depth \cite{kaelbling1998planning}. Such exponential growth is typically attributed to combinatorial explosion in state or policy spaces, where the number of feasible trajectories increases rapidly over time.

To mitigate these challenges, the RL literature has introduced various approaches, including temporal abstraction and the options framework \cite{sutton1999between}, which reduce effective decision depth through macro-actions or hierarchical decomposition. However, these methods primarily interpret long-horizon difficulty through the lens of search complexity or policy representation. They generally do not analyze how uncertainty accumulates dynamically when a generative model repeatedly performs autoregressive updates along a single execution trajectory.

Theorem~A proposed in this work exhibits mathematically similar exponential behavior, but differs fundamentally in interpretation. Rather than attributing failure to the proliferation of possible paths, we show that collapse can arise even in strictly linear settings without branching, due to cumulative noise within the autoregressive generative process itself.

\subsection{Chain-Based Reasoning in Large Language Models}

Recent advances in large language models (LLMs) have led to a surge of techniques designed to explicitly guide reasoning processes. Chain-of-Thought (CoT) prompting encourages models to generate intermediate reasoning steps, significantly improving performance on arithmetic and logical tasks \cite{wei2022chain}. Subsequent approaches such as Tree-of-Thought (ToT) introduce branching exploration and backtracking mechanisms to evaluate multiple reasoning candidates \cite{yao2023tree}, while related frameworks combine reasoning and acting in interactive environments \cite{yao2023react}.

A common characteristic of these methods is the introduction of structure at the semantic level, aiming to improve output quality through explicit intermediate representations. Nevertheless, most existing work implicitly treats reasoning as an autoregressive sequence that can be extended indefinitely. The physical or informational constraints governing the stability of long reasoning trajectories remain largely unformalized. Consequently, even semantically coherent intermediate steps may lead to logical drift or inconsistency as reasoning depth increases.

Our perspective does not challenge the empirical effectiveness of CoT or ToT. Instead, we argue that their success implicitly relies on segmentation and reset mechanisms within the reasoning process. In our framework, this phenomenon emerges as a structural constraint on edge length within a reasoning graph, supported by an information-theoretic analysis.

\subsection{Structured Reasoning and Graph-Based Representations}

Another line of research represents reasoning as graph structures, such as directed graphs or directed acyclic graphs (DAGs), to explicitly model dependencies between states, subproblems, or intermediate conclusions. These approaches have been applied in planning, program synthesis, and knowledge reasoning, where complex tasks are decomposed into manageable components. Recent work such as Graph-of-Thoughts extends this paradigm by organizing reasoning as structured graph traversal \cite{besta2024graph}.

However, many structured reasoning methods assume that nodes correspond to stable, randomly accessible representations. In generative models, by contrast, such nodes are themselves produced through sequential autoregressive processes and are therefore constrained by limited context capacity and accumulated uncertainty. Existing graph-based reasoning frameworks primarily model structure at the semantic or symbolic level, without addressing the dynamical limitations inherent in the generation process.

In our formulation, the DAG is not merely a representational abstraction but a topological consequence of stability constraints. Edges correspond to bounded segments of autoregressive execution, while nodes correspond to necessary operations of state consolidation and informational compression.

\subsection{Computational Argumentation and Dialogue Systems}

Recent research has also explored the integration of computational argumentation into dialogue systems, formalizing relationships between claims, rebuttals, and supporting evidence. Argumentation-based conversational agents offer potential advantages in interpretability and non-monotonic reasoning, and can be combined with Transformer architectures \cite{vaswani2017attention} to enhance dialogue quality.

However, this line of work primarily focuses on structuring reasoning content rather than analyzing the stability of long generative trajectories. Even if argument structures are logically valid at an abstract level, their generation may still suffer from instability when executed as extended autoregressive sequences. Thus, computational argumentation and our framework operate at complementary levels: the former addresses semantic representation and evaluation, whereas the latter addresses whether reasoning execution can be sustained dynamically under information-theoretic constraints.

\subsection{Summary}

In summary, prior work has proposed numerous methods for improving long-horizon reasoning through search strategies, semantic guidance, structured representations, or argumentation theory. Nevertheless, the intrinsic stability limits of autoregressive generation remain underexplored. Through Theorem~A and the subsequent DAG-based reasoning framework, this work aims to fill this gap by providing a process-level analysis grounded in information theory and dynamical systems perspectives, offering a structural interpretation of ``System 2'' reasoning behaviors \cite{kahneman2011thinking}.

\section{Process-Level Stability Limits of Long-Horizon Autoregressive Reasoning (Theorem~A)}
\label{sec:theoremA}

This chapter addresses a question that has remained largely implicit in prior work: \emph{without introducing any structural interruption, can an autoregressive reasoning trajectory be stably extended in an information-theoretic sense?}

Most existing approaches to long-horizon reasoning focus on mitigating task-level difficulties---search complexity, semantic structure, or proof formats---and implicitly assume that once the reasoning content is properly guided, the execution process can be prolonged as needed. For autoregressive language models, however, reasoning does not unfold over a static, randomly accessible state. Instead, it is realized through a sequential execution process whose implicit state evolves over time while accumulating uncertainty. Even in linear tasks without branching search, backtracking, or semantic ambiguity, it remains unclear whether such an execution process can be extended indefinitely.

In what follows, we show that long-horizon autoregressive reasoning is subject to intrinsic stability limits even under highly simplified conditions: no search branching, a known policy, and minimal task structure. We formalize this phenomenon as Theorem~A, and in Section~\ref{sec:structural_consequence} discuss its direct structural consequences for reasoning topology.

\begin{figure*}[t]
\centering
\begin{tikzpicture}[
    node/.style={draw, rounded corners=2pt, minimum width=10mm, minimum height=6mm, align=center},
    small/.style={font=\small},
    arr/.style={-Latex, line width=0.6pt},
    sep/.style={line width=0.8pt, dashed},
]
\node[small] at (0,2.3) {\textbf{(a) Single long edge: noise accumulation erodes directionality}};
\node[node] (z0) at (-3,1.2) {$Z_0$};
\node[node] (z1) at (-1.5,1.2) {$Z_1$};
\node[node] (z2) at (0,1.2) {$Z_2$};
\node[node] (z3) at (1.5,1.2) {$Z_3$};
\node[node] (zL) at (3.2,1.2) {$Z_L$};

\draw[arr] (z0) -- (z1);
\draw[arr] (z1) -- (z2);
\draw[arr] (z2) -- (z3);
\draw[arr] (z3) -- (zL);

\node[small] at (0,0.55) {$\rho(L)\ \le\ \rho_0 e^{-\gamma L}$};

\draw[line width=0.6pt] (-3,0.1) -- (3.2,0.1);
\draw[line width=2.0pt] (-3,0.1) -- (-1.2,0.1); 
\draw[line width=1.2pt] (-1.2,0.1) -- (0.8,0.1);
\draw[line width=0.6pt] (0.8,0.1) -- (3.2,0.1);

\node[small, anchor=west] at (-3.0,-0.25) {Directionality (decision advantage) decays with length};
\node[small, anchor=west] at (1.1,-0.55) {$L > L^*$ implies $\rho < \tau$};

\begin{scope}[xshift=9.2cm]
\node[small] at (0,2.3) {\textbf{(b) Segmented execution: interruptions suppress accumulation}};
\node[node] (a0) at (-3,1.2) {$Z_0$};
\node[node] (a1) at (-1.5,1.2) {$Z_1$};
\node[node] (a2) at (0,1.2) {$Z_2$};

\node[node] (b0) at (1.8,1.2) {$Z'_0$};
\node[node] (b1) at (3.3,1.2) {$Z'_1$};
\node[node] (b2) at (4.8,1.2) {$Z'_2$};

\draw[arr] (a0) -- (a1);
\draw[arr] (a1) -- (a2);

\draw[arr] (b0) -- (b1);
\draw[arr] (b1) -- (b2);

\draw[sep] (0.9,1.75) -- (0.9,0.65);
\node[small, rotate=90] at (0.75,1.2) {reset};

\node[small] at (0.9,0.55) {Break history dependence and suppress accumulation};

\draw[line width=0.6pt] (-3,0.1) -- (0,0.1);
\draw[line width=2.0pt] (-3,0.1) -- (-1.8,0.1);
\draw[line width=1.2pt] (-1.8,0.1) -- (0,0.1);

\draw[line width=0.6pt] (1.8,0.1) -- (4.8,0.1);
\draw[line width=2.0pt] (1.8,0.1) -- (3.0,0.1);
\draw[line width=1.2pt] (3.0,0.1) -- (4.8,0.1);

\node[small, anchor=west] at (-3.0,-0.25) {Each segment satisfies: $\ell_i < L^*$};
\end{scope}

\end{tikzpicture}
\caption{A process-level view of degradation in autoregressive reasoning. Even in linear tasks without branching search, a single long execution edge exhibits exponential decay of decision advantage due to accumulated uncertainty; in contrast, segmentation with resets suppresses accumulation, keeping each segment within a stable regime.}
\label{fig:process_stability}
\end{figure*}

\subsection{Problem Setup: Reasoning Execution Rather Than Learning}
\label{sec:problem_setup}

We consider a fixed, fully trained autoregressive model executing a reasoning process of length $L$ along a single trajectory under a given task and a fixed reasoning policy. In this setting:

\begin{itemize}
    \item No cross-episode learning or policy updates are involved;
    \item Questions of policy optimality or search efficiency are outside the scope of analysis;
    \item The central inquiry is whether the execution of reasoning can stably maintain directional alignment toward the correct objective as the trajectory length increases.
\end{itemize}

This perspective complements the analysis presented in Appendix~A. While Appendix~A studies learning-theoretic sample complexity in sparse long-horizon tasks (TW-HSF$_\varepsilon$), demonstrating exponential growth in required samples when identifiable structural signals are absent, the present chapter focuses on a different question: even assuming learning has been completed and the reasoning policy is fixed, can a single execution trajectory maintain its reliability as the reasoning horizon extends?

\subsection{Decision Advantage and Noise Accumulation}
\label{sec:decision_advantage}

We characterize the \emph{directionality} of a reasoning trajectory from an information-theoretic perspective. Let $G$ denote the target proposition (the correct conclusion), and $\neg G$ its negation. We define $Z_t$ as the internal latent state of the model at step $t$, which encapsulates the accumulated context and implicit representations up to that point.

\begin{definition}[Decision Advantage]
The decision advantage at step $t$, denoted by $\rho_t \in [-1, 1]$, is defined as:
\begin{equation}
\rho_t \triangleq \mathbb{P}(G \mid Z_t) - \mathbb{P}(\neg G \mid Z_t).
\end{equation}
\end{definition}

A positive value $\rho_t > 0$ indicates that the model remains directionally aligned with the target objective. Conversely, $\rho_t \approx 0$ signifies a loss of effective directionality, where the model enters a regime of high uncertainty or "hallucination," making correct recovery increasingly improbable.

We model the autoregressive update as a stochastic dynamical process:
\begin{equation}
Z_{t+1} = f(Z_t) + \varepsilon_t,
\end{equation}
where $f$ represents the ideal transition function of the reasoning policy, and $\varepsilon_t$ denotes the accumulated noise—comprising approximation errors, sampling variance, and representational decay. 

The central premise of our analysis is that in long-horizon sequences, these perturbations $\varepsilon_t$ typically accumulate rather than cancel out, leading to a monotonic erosion of the decision advantage.


\subsection{Theorem A: Process Collapse in Long-Horizon Autoregressive Reasoning}
\label{sec:theorem_statement}
We analyze long-horizon autoregressive reasoning under a set of idealized structural assumptions. 
The following result should be interpreted as a structural stability theorem under contraction-like dynamics, 
rather than a tight characterization of any specific model architecture.

\begin{itemize}
    \item \textbf{Finite representational capacity.} The internal state admits bounded informational precision, preventing lossless accumulation of arbitrarily long histories.

    \item \textbf{Persistent stochastic perturbations.} Each autoregressive update introduces non-zero uncertainty, arising from approximation error, sampling noise, or representational aliasing.

    \item \textbf{Single-path execution.} Reasoning proceeds along a fixed trajectory without branching search, external reset mechanisms, or explicit backtracking.

    \item \textbf{Absence of structural grounding.} Intermediate states lack external structural signals capable of fully eliminating accumulated uncertainty.
\end{itemize}

These assumptions capture common constraints in practical autoregressive reasoning systems rather than properties of any specific model architecture.

We model the reasoning dynamics as a stochastic autoregressive process:
\begin{equation}
Z_{t+1} \sim K(\cdot \mid Z_t),
\end{equation}
where $K$ denotes the induced transition kernel governing state evolution.

\begin{theorem}[Process Collapse in Long-Horizon Autoregressive Reasoning]
Under mild contraction assumptions on the autoregressive transition kernel, suppose the stochastic transition kernel admits a contraction coefficient $\eta < 1$ in total variation distance. Then the decision advantage along a single-path reasoning trajectory of length $L$ satisfies
\begin{equation}
\rho(L) \le \rho_0 e^{-\gamma L},
\end{equation}
where $\gamma = -\ln \eta > 0$ and $\rho_0$ denotes the initial decision advantage.

Consequently, there exists a critical length
\begin{equation}
L^* = \frac{1}{\gamma} \ln \frac{\rho_0}{\tau},
\end{equation}
such that for $L > L^*$ the decision advantage falls below a reliability threshold $\tau$, beyond which directional inference becomes unstable.
\end{theorem}

\paragraph{Proof sketch.}
Let $\Delta_t$ denote the distinguishability between state distributions conditioned on correct versus incorrect reasoning trajectories, measured in total variation distance. Persistent stochastic perturbations imply that the induced Markov operator exhibits contraction:

\begin{equation}
\Delta_{t+1} \le \eta \, \Delta_t,
\end{equation}

where $\eta < 1$ reflects cumulative uncertainty introduced at each autoregressive step. Iterating yields

\begin{equation}
\Delta_t \le \eta^t \Delta_0.
\end{equation}

Since the decision advantage $\rho_t$ is bounded by distributional distinguishability, i.e.,
\begin{equation}
\rho_t \le \Delta_t,
\end{equation}
we obtain exponential decay:

\begin{equation}
\rho_t \le \rho_0 e^{-\gamma t}.
\end{equation}

Detailed derivations and alternative formulations based on entropy accumulation are provided in Appendix~A.3. The contraction assumption is not derived here but serves as an idealized abstraction capturing cumulative uncertainty effects in autoregressive execution.

\paragraph{Dynamical interpretation.}
The exponential decay bound reflects a contraction-like behavior in decision advantage under noisy autoregressive dynamics. Persistent perturbations progressively reduce directional information, implying that sufficiently long single-path execution inevitably approaches a high-uncertainty regime even in the absence of branching complexity.

\subsection{Distinction from Problem-Complexity Lower Bounds}
\label{sec:complexity_distinction}

In classical reinforcement learning and planning frameworks, long-horizon failure is predominantly attributed to problem complexity. For instance, search spaces expand exponentially with branching factor $b$, leading to exploration and credit-assignment challenges that scale as $b^L$. Within this paradigm, reasoning failure is primarily interpreted as a consequence of combinatorial explosion.

In contrast, Theorem~A identifies an orthogonal source of failure:

\begin{itemize}
    \item When $b = 1$ (i.e., purely linear tasks), search complexity effectively vanishes;
    \item Nevertheless, even in the absence of branching, the decision advantage decays exponentially with reasoning horizon $L$ due to accumulation of autoregressive noise.
\end{itemize}

Therefore, long-horizon reasoning difficulty cannot be fully explained by problem complexity alone. Even in tasks where the branching factor equals one, autoregressive execution remains subject to intrinsic process-level instability. This theoretical distinction is further supported by the Chain Sensitivity analysis in Section~5, which provides empirical evidence for the existence of a process-level bottleneck.

\subsection{Structural Consequence: Why Stable Reasoning Requires Discretization}
\label{sec:structural_consequence}

The direct implication of Theorem~A is not the fundamental impossibility of long-horizon reasoning, but rather the infeasibility of sustaining it through an indefinitely extended single linear autoregressive trajectory.

As established by Theorem~A, once the reasoning length $L$ exceeds the critical scale $L^*$, the decision advantage $\rho(L)$ falls below the threshold $\tau$ required for reliable logical discrimination. If reasoning continues along a single trajectory, the process inevitably enters a regime characterized by dissipation of directional information and increasing dominance of stochastic noise.

To avoid such process-level collapse, reasoning execution must admit a partition into discrete segments such that each continuous autoregressive edge $\ell_i$ remains within the stable regime:
\begin{equation}
\ell_i < L^*.
\end{equation}

This constraint implies a structural transformation: stable long-horizon reasoning can no longer be represented as a monolithic linear chain. Instead, it evolves into multiple finite-length edges interconnected by stabilizing nodes that perform state consolidation or informational reset. Topologically, these requirements naturally correspond to DAG-like structures.

Importantly, this discretization is not merely an engineering preference but emerges as a structural consequence of maintaining stability under autoregressive dynamics. In this framework, DAG-like reasoning structures arise as a natural organizational form shaped by informational stability constraints.

\subsection{Critical Length $L^*$ as a Diagnostic Scale for Reasoning Stability}
\label{sec:critical_length}

Existing literature frequently characterizes long-horizon reasoning failure through qualitative observations, such as coherence degradation, increased hallucination, or logical drift. While these descriptions capture empirical phenomena, they lack a predictive framework and do not address the precise boundary conditions under which failure becomes inevitable.

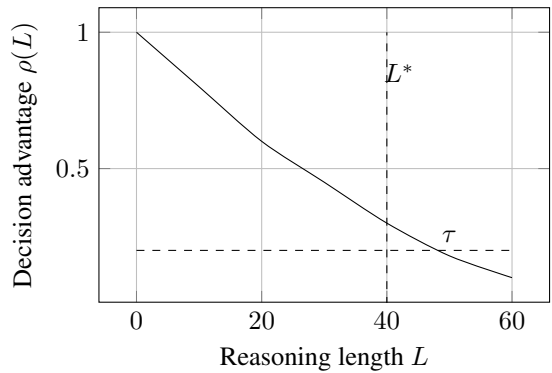
\begin{figure}[t]
\centering
\begin{tikzpicture}
\begin{axis}[
    width=0.9\linewidth,
    height=5.5cm,
    xlabel={Reasoning length $L$},
    ylabel={Decision advantage $\rho(L)$},
    grid=major,
]

\addplot[smooth] coordinates {
(0,1)
(10,0.8)
(20,0.6)
(30,0.45)
(40,0.3)
(50,0.18)
(60,0.1)
};

\addplot[dashed] coordinates {(0,0.2) (60,0.2)};
\node at (axis cs:50,0.25) {$\tau$};

\draw[dashed] (axis cs:40,0) -- (axis cs:40,1);
\node at (axis cs:42,0.85) {$L^*$};

\end{axis}
\end{tikzpicture}
\caption{
Conceptual phase transition induced by intrinsic stability limits.
Decision advantage decays exponentially and crosses the critical threshold at $L^*$.
}
\label{fig:phase_transition}
\end{figure}

As illustrated in Figure~\ref{fig:phase_transition}, the decision advantage crosses the critical threshold at $L^*$, marking the onset of instability.

Theorem~A addresses this gap by introducing the \emph{critical length} $L^*$, which provides a quantitative diagnostic scale for reasoning stability. Derived from our analysis in Section~3, $L^*$ is characterized by the intrinsic decay rate $\gamma$ and the initial decision advantage $\rho_0$:
\begin{equation}
L^* = \frac{1}{\gamma} \ln \frac{\rho_0}{\tau}.
\end{equation}

This result allows the sustainability of reasoning trajectories to be interpreted as an intrinsic dynamical property rather than purely an empirical observation. Concretely, $\gamma$ captures the average rate at which goal-relevant information dissipates during autoregressive updates, while $\rho_0$ reflects the initial directional clarity of the process. Together, they define a fundamental stability horizon. When the reasoning length $L$ approaches or exceeds $L^*$, process-level instability dominates, driving the system toward a high-uncertainty regime even in tasks with minimal structural complexity.

From this perspective, $L^*$ serves as a \textbf{stability ruler} for autoregressive inference. Moving beyond retrospective explanation, it offers a principled criterion for anticipating potential failure points of single-path execution under specific model and task configurations.

Importantly, the introduction of $L^*$ also suggests structural constraints on stabilization strategies: to sustain coherent long-horizon reasoning, mechanisms must ensure that elementary reasoning segments remain below this critical scale. This insight provides a theoretical foundation for reinterpreting and optimizing structured reasoning paradigms.

\subsection{Summary}
\label{sec:chapter3_summary}

In this chapter, we characterized long-horizon autoregressive reasoning as a dynamical process and established the existence of an intrinsic \emph{process-level instability}. We demonstrated that even in the absence of branching complexity, the decision advantage decays exponentially with reasoning length, leading to the emergence of a fundamental stability horizon, $L^*$.

This result implies a structural requirement: stable long-horizon reasoning cannot be sustained through a monolithic, uninterrupted trajectory. To preserve logical coherence, reasoning execution must admit a partition into discrete segments such that each continuous edge $\ell_i$ satisfies the stability criterion:
\begin{equation}
\ell_i < L^*.
\end{equation}

This constraint provides the theoretical basis for transitioning from linear chains toward structured graph-like representations. In the following chapter, we expand this framework by examining the functional roles of reasoning nodes—specifically their capacity for state consolidation and entropy regulation. We further analyze how node-based operations suppress the accumulation of historical uncertainty, enabling reliable reasoning beyond the critical scale $L^*$.

\section{Inference Segmentation Necessity and Structural Implications}
\label{sec:chapter4}

This chapter synthesizes the theoretical foundations of Theorem~A (Chapter~3) with the empirical findings presented in Chapter~5, establishing a unified structural framework for understanding long-horizon reasoning.

We contend that the failure of long-horizon reasoning cannot be attributed solely to problem complexity. Beyond the conventional focus on combinatorial search or credit-assignment challenges, intrinsic stability limits—arising from the dynamics of autoregressive execution—play a fundamental role. From this perspective, structured reasoning techniques such as Chain-of-Thought \cite{wei2022chain}, Tree-of-Thought \cite{yao2023tree}, and Graph-of-Thought \cite{besta2024graph}, traditionally introduced as heuristic improvements, can instead be interpreted as structural responses to stability constraints. Structural segmentation and reset mechanisms therefore emerge not merely as engineering choices, but as necessary conditions for preventing process-level collapse.

By integrating rigorous theoretical analysis with empirical validation, we examine how these stability constraints shape the structural requirements of reasoning execution. This analysis motivates a transition in reasoning architectures: moving away from monolithic linear trajectories toward segmented, topology-aware configurations capable of sustaining coherence over extended horizons.

\subsection{Failure Mode Taxonomy for Long-Horizon Reasoning}
\label{sec:failure_taxonomy}

Based on the theoretical and empirical findings of this study, failures in long-horizon reasoning can be organized into three mutually orthogonal sources:

\begin{enumerate}
\item \textbf{Problem Complexity (Combinatorial Explosion).} 
When the branching factor $b$ of a task scales with the reasoning horizon $L$, the required computational and sample complexity increases exponentially ($b^L$). Our synthetic Track~B experiments and the analysis in Appendix~A illustrate this classical lower bound. The origin of this failure mode lies in the combinatorial explosion of feasible trajectories, a phenomenon extensively documented in planning and reinforcement learning literature \cite{kaelbling1998planning,sutton1999between}.

\item \textbf{Process Instability (Information-Theoretic Decay).} 
Critically, even in strictly linear settings ($b=1$), long-horizon autoregressive execution accumulates internal decision noise and residual uncertainty. As demonstrated by our Chain Sensitivity analysis, the decision advantage $\rho(L)$ is predicted to decay exponentially with reasoning length under Theorem~A. This mode is fundamentally orthogonal to search complexity and arises from intrinsic dynamical instability in autoregressive execution.

\item \textbf{Ecological Traps (Dynamical Stagnation).} 
In complex interactive environments, the reasoning process may further be compromised by historical inertia, local attractors, or oscillatory regimes. Our Track~A TextWorld experiments \cite{cote2018textworld} show that unstructured progression frequently falls into such ecological traps, where reasoning stagnates despite the sufficiency of the model's underlying knowledge.
\end{enumerate}

This taxonomy underscores that the feasibility of long-horizon reasoning is governed not only by task difficulty but also by the intrinsic stability of the execution process. Consequently, addressing problem complexity through search or scaling alone is insufficient to guarantee sustained reasoning performance in extended horizons.

\subsection{Topology Under Critical Edge-Length Constraints (DAG as a Canonical Realization)}
\label{sec:topology_constraint}

Theorem~A in Chapter~3 establishes that a single autoregressive reasoning edge possesses an intrinsic critical length $L^*$ determined by the exponential decay rate of decision advantage. Empirical results in Chapter~5 further demonstrate that structured segmentation and reset mechanisms shift the effective failure scale from the global reasoning horizon to local segment lengths. Taken together, these findings suggest that reasoning processes aiming to maintain stability over extended horizons must avoid unbounded single-path extension.

Consequently, under the stability constraints defined in Chapter~3, any robust reasoning topology must satisfy
\begin{equation}
\ell_i < L^*,
\end{equation}
for each reasoning edge $\ell_i$. Violating this condition drives the system toward a high-uncertainty regime as decision advantage decays beyond the stability threshold. This constraint introduces a \textbf{minimum density of stabilization nodes} along the reasoning horizon, motivating a transition from monolithic linear trajectories toward structured graph-based representations observed in prior structured reasoning approaches such as Chain-of-Thought \cite{wei2022chain}, Tree-of-Thought \cite{yao2023tree}, and Graph-of-Thought \cite{besta2024graph}.

From a topological perspective, directed acyclic graphs (DAGs) provide a canonical realization of structures that satisfy local stability constraints while preserving global reachability. In this framework:

\begin{itemize}
\item \textbf{Edges} correspond to finite-length autoregressive execution segments constrained by $\ell_i < L^*$;
\item \textbf{Vertices} correspond to stabilization points where state consolidation, compression, or reset operations occur;
\item \textbf{Acyclicity} prevents recursive amplification of historical uncertainty through closed feedback loops.
\end{itemize}

Within this interpretation, nodes function as informational anchors analogous to information bottlenecks \cite{tishby2000information}. Their role is to project accumulated history into compressed yet task-relevant representations, filtering high-entropy noise while preserving goal-relevant mutual information. From this perspective, node formation emerges as a structural response to entropy accumulation inherent in autoregressive dynamics rather than as a purely heuristic engineering choice.

If long-horizon reasoning is to remain stable, execution must admit decomposition into sub-critical segments whose effective lengths satisfy $\ell_i < L^*$. Such segmentation naturally induces graph-like reasoning topologies, among which DAGs represent a canonical and practically effective implementation. Importantly, DAGs are not the only possible realization; rather, they serve as representative structures emerging from edge-length constraints and the necessity of reset mechanisms.

\subsection{Quantitative Structural Interpretation of System-2 Behavior: From Heuristic Techniques to Critical Response}
\label{sec:system2_structural}

In psychology and artificial intelligence literature, ``System~2'' reasoning is commonly used to describe slow, deliberate, and reflective cognitive processes \cite{kahneman2011thinking}. Despite its widespread adoption, System~2 behavior has largely been treated as a qualitative phenomenon, lacking a unified quantitative explanation regarding when and why such behavior emerges.

Within the framework developed in this work, System~2 behavior can be reinterpreted as a structural response arising when a single-path reasoning trajectory approaches the critical scale $L^*$. As linear autoregressive execution nears this intrinsic stability limit, maintaining decision advantage requires structural adaptation. From this perspective, System~2 is not merely a cognitive style but can be understood as a form of \emph{structural phase transition} that occurs when linear reasoning approaches its stability boundary.

To characterize this critical response, we introduce three complementary but non-essential diagnostic quantities. These correspond respectively to dynamical degradation, critical transition signals, and structural scale constraints, forming a multi-scale diagnostic framework spanning dynamical and topological levels.

\paragraph{(1) Decision Advantage Decay Rate.}
\begin{equation}
\gamma_t = -\frac{1}{t}\ln\frac{\rho_t}{\rho_0}.
\end{equation}
When $\gamma_t$ significantly exceeds the intrinsic decay rate $\gamma$, it indicates accelerated loss of directional information, suggesting that the reasoning trajectory is approaching its stability limit.

\paragraph{(2) Conditional Entropy Spike Indicator.}
\begin{equation}
\Delta H_t = H(X_{t+1} \mid H_t) - H(X_{t+1} \mid \mathcal{T}(H_t, s^*)).
\end{equation}
When $\Delta H_t$ exceeds its typical background distribution, accumulated historical noise begins to dominate predictive capacity, implying diminishing or even negative marginal returns from extending the reasoning trajectory.

\paragraph{(3) Effective Edge-Length Ratio.}
\begin{equation}
r_t = \frac{t - t_{\text{last reset}}}{L^*}.
\end{equation}
For instance, when $r_t > 0.8$, single-path reasoning enters a highly unstable regime in which attempts to maintain consistency solely by extending the sequence are likely to fail rapidly.

Together, these indicators define a computable critical region. When any two exceed their respective thresholds, the reasoning process can be interpreted as entering a high-instability regime. Importantly, these quantities are not proposed as control mechanisms or precise measurement tools; rather, they function as diagnostic signals for identifying when linear reasoning approaches a dynamically unsustainable region.

Under this interpretation, existing ``slow-thinking'' techniques can be reframed as different responses to the critical scale $L^*$. Chain-of-Thought prompting \cite{wei2022chain} primarily delays collapse by introducing semantic constraints that effectively reduce the decay rate $\gamma$. However, because the underlying topology remains linear, collapse becomes increasingly likely as $r_t \rightarrow 1$. In contrast, Tree-of-Thought approaches \cite{yao2023tree} introduce explicit branching or structural intervention before the critical boundary is crossed, effectively partitioning reasoning into segments satisfying $\ell_i < L^*$.

Consequently, within this framework, System~2 behavior emerges not as an ad hoc engineering heuristic but as a quantifiable structural response to the critical length $L^*$. The diagnostic quantities ($\gamma_t$, $\Delta H_t$, $r_t$) primarily serve as interpretive tools bridging theory and empirical observation. Precise numerical estimation and large-scale operational deployment remain important directions for future work.

\subsection{Scope and Limitations}
\label{sec:limitations}

To ensure theoretical rigor and avoid overgeneralization, we explicitly delineate the scope and boundary conditions of the present work:

\begin{enumerate}
\item \textbf{Node Generation and Algorithmic Optimization.} 
This study characterizes the functional necessity of stabilization nodes but does not prescribe specific algorithms for their automated generation or optimization. The mechanisms governing adaptive segmentation or optimal node placement remain compelling directions for future research.

\item \textbf{Optimality of Segmentation Strategies.} 
Our analysis establishes the \emph{structural necessity} of segmentation and reset mechanisms for maintaining informational stability. However, we do not investigate the mathematical optimality of particular segmentation policies or the associated computational overhead of frequent state consolidation.

\item \textbf{Generalization Beyond Single-Agent Autoregressive Execution.} 
While scenarios involving multi-agent collaboration, tool invocation, or external memory architectures may be conceptually congruent with the proposed framework, such extensions remain beyond the scope of this foundational study.
\end{enumerate}

These limitations clarify the functional boundaries of our inquiry but do not undermine the central theoretical conclusion: the existence of intrinsic stability constraints that govern long-horizon autoregressive reasoning.

\subsection{Summary}
\label{sec:chapter4_summary}

By synthesizing the theoretical analysis of Chapter~3 with the structural framework developed in this chapter, we arrive at a rigorous conclusion: if long-horizon reasoning relies on the continuous extension of a single autoregressive trajectory, its stability is inherently bounded by a critical length $L^*$ under the stated assumptions. Any strategy that attempts to extend reasoning indefinitely without structural intervention will ultimately lead to the collapse of decision advantage. Consequently, the transition from linear chains toward DAG-like or graph-structured execution—featuring bounded edge lengths and stabilization nodes—should be understood not merely as an engineering heuristic, but as a structural consequence induced by the dynamics of autoregressive instability.

Within this framework, structured reasoning topologies give rise to a set of empirically testable structural predictions that transcend specific model architectures:

\begin{enumerate}
\item \textbf{Existence of Performance Cliffs:} Reasoning performance is expected to exhibit sharp degradation once single-path execution approaches the critical scale $L^*$;

\item \textbf{Efficacy of Segmentation:} Partitioning execution into sub-critical edges ($\ell_i < L^*$) and introducing informational resets should significantly mitigate the decay of decision advantage;

\item \textbf{Cross-Domain Consistency:} These effects should manifest consistently across synthetic benchmarks, mechanistic analyses, and complex interactive environments, largely independent of semantic complexity or specific task formulations.
\end{enumerate}

Chapter~5 is explicitly designed to validate these structural predictions. Rather than discovering structure through empirical observation alone, we employ targeted evaluations to test necessary conditions derived from theoretical analysis. Through complementary experiments—including synthetic long-horizon tasks, branch-free chain sensitivity analysis, and TextWorld benchmarks \cite{cote2018textworld}—we demonstrate that adherence to these structural constraints plays a critical role in sustaining reasoning stability. Conversely, ignoring such constraints leads to systematic degradation patterns consistent with the instability framework.

In summary, the structural conclusions established here characterize the \emph{necessary conditions} for reliable long-horizon autoregressive inference. While the design of optimal or fully operational stabilization mechanisms remains an open research direction, the requirement for bounded linear extension establishes a principled structural constraint that underlies the broader framework of \emph{structural governance} developed in this work.

\section{Experimental Evaluation}
\label{sec:experiments}

This chapter presents a set of complementary experiments designed to empirically examine the central thesis introduced in Chapter~3: that failures in long-horizon reasoning arise not solely from task-level complexity, but also from intrinsic stability limits inherent to autoregressive execution. Furthermore, we investigate whether structural segmentation and informational reset mechanisms can mitigate these failure modes under stability constraints derived from our theoretical analysis.

To align theoretical predictions with observable behavior, we adopt the diagnostic framework introduced in Section~4.3. Rather than treating instability as a monolithic failure phenomenon, we interpret long-horizon degradation as the manifestation of identifiable dynamical signals. This perspective enables empirical observations to be analyzed through a process-level lens, bridging theoretical dynamics with measurable behavioral outcomes.

The experimental evaluation is organized into three distinct tracks, each targeting a different level of validation:

\begin{itemize}
\item \textbf{Track B (Synthetic Tasks):} Evaluates the exponential sample-complexity lower bounds derived in Appendix~A. These experiments demonstrate how structural segmentation shifts the effective computational bottleneck from the global horizon $L$ to the longest effective sub-critical edge length $\ell_{\max}$.

\item \textbf{Chain Sensitivity Analysis (Mechanistic Study):} Isolates collapse behavior arising purely from process instability within a branch-free environment. By eliminating search complexity and semantic ambiguity, this setup provides a controlled evaluation of the exponential decay predicted by Theorem~A.

\item \textbf{Track A (TextWorld Environments):} Examines structural governance mechanisms in interactive language-based environments \cite{cote2018textworld}, assessing whether topology-aware reasoning structures improve forward progress and suppress oscillatory behaviors under realistic conditions.
\end{itemize}

Within the diagnostic framework, we utilize three quantitative indicators to characterize critical behavior: decision advantage decay ($\gamma_t$), conditional entropy spikes ($\Delta H_t$), and the effective edge-length ratio ($r_t$). These quantities are not proposed as optimization objectives or control variables; rather, they function as \emph{observational instruments} for detecting when reasoning trajectories approach the instability regime predicted by Theorem~A.

The primary objective of this evaluation is therefore not the direct numerical validation of these indicators, but to assess whether reasoning behavior exhibits dynamical patterns consistent with theoretical predictions. Specifically, we examine:

\begin{enumerate}
\item exponential-scale transitions in synthetic long-horizon tasks;
\item entropy-driven collapse phenomena in branch-free linear chains;
\item improved stability in interactive environments when reasoning segments are constrained below the critical scale $L^*$.
\end{enumerate}

Under this diagnostic lens, the experimental results presented in the following sections should be understood as empirical projections of the theoretical critical scales and dynamical boundaries developed in the preceding chapters.

\subsection{Track B: Sample Complexity Analysis in Synthetic Long-Horizon Tasks}
\label{sec:trackB}

\subsubsection{Experimental Setup}

Guided by the theoretical framework of the TextWorld Hard Sparse Family (TW-HSF$_\varepsilon$) defined in Appendix~A, we construct a synthetic long-horizon task environment designed to isolate the structural bottleneck of sparse reasoning. The environment enables high-throughput sampling while preserving the essential property that success requires the exact execution of a unique action sequence of length $L$. Any deviation collapses the trajectory into failure, and no intermediate feedback is provided.

Under the \textbf{unstructured baseline}, the success probability per episode is analytically characterized by
\begin{equation}
p = |\mathcal{A}|^{-L},
\end{equation}
where $|\mathcal{A}|$ denotes the action space size. Consequently, the episodes-to-success follow a geometric distribution whose expected value grows exponentially with $L$, consistent with the theoretical lower bound derived in Appendix~A.

In the \textbf{structured condition (Landmarks)}, the global horizon $L$ is partitioned into $K$ subsegments $(\ell_1, \dots, \ell_K)$, each independently solvable. To simulate imperfect structural signals, we introduce a \emph{landmark omission probability} $p_{\mathrm{drop}}$, whereby failed landmark detection stochastically merges adjacent segments, thereby increasing the effective uninterrupted reasoning span and raising the longest effective segment length $h_{\max}$.

\subsubsection{Results and Analysis}

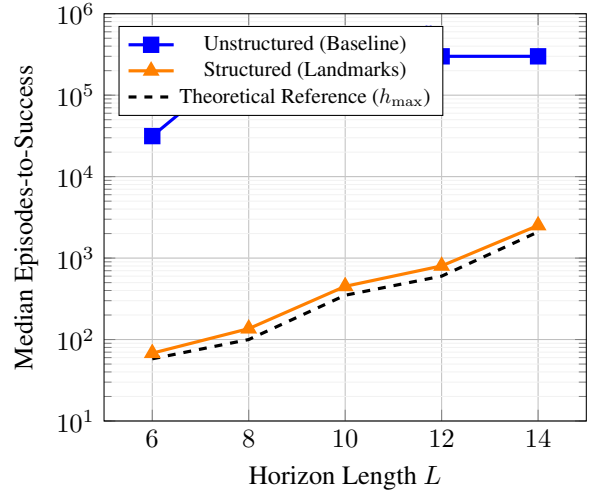
\begin{figure}[t]
  \centering
 \begin{tikzpicture}
\begin{semilogyaxis}[
    width=0.95\linewidth,
    height=7cm,
    xlabel={Horizon Length $L$},
    ylabel={Median Episodes-to-Success},
    xmin=5, xmax=15,
    ymin=10, ymax=1000000,
    xtick={6,8,10,12,14},
    grid=both,
    grid style={line width=.1pt, draw=gray!10},
    major grid style={line width=.2pt, draw=gray!50},
    legend pos=north west,
    legend style={nodes={scale=0.8, transform shape}},
    every axis plot/.append style={line width=1.2pt},
    mark size=2.5pt
]

\addplot[color=blue, mark=square*] coordinates {
    (6, 31452)
    (8, 300000)
    (10, 300000)
    (12, 300000)
    (14, 300000)
};
\addlegendentry{Unstructured (Baseline)}

\addplot[color=orange, mark=triangle*] coordinates {
    (6, 68)
    (8, 136)
    (10, 450)
    (12, 802)
    (14, 2515)
};
\addlegendentry{Structured (Landmarks)}

\addplot[color=black, dashed, no marks] coordinates {
    (6, 58)
    (8, 100)
    (10, 350)
    (12, 600)
    (14, 2096)
};
\addlegendentry{Theoretical Reference ($h_{\max}$)}

\node[blue, font=\small] at (axis cs: 10, 500000) {Capped at $3 \times 10^5$};

\end{semilogyaxis}
\end{tikzpicture}
  \caption{
  \textbf{Track B: Exponential bottleneck transfer under structural segmentation.}
  Median episodes-to-success (log scale) as a function of horizon length $L$ for unstructured (blue) and structured / landmark-based (orange) agents in the synthetic TW-HSF setting.
  The dashed curve shows the theoretical reference scale dominated by the longest effective segment length $h_{\max}$.
  Without structure, sample complexity grows exponentially in $L$, consistent with the $|\mathcal{A}|^{L}$ lower bound (Appendix~A).
  Structural segmentation reconfigures the exponential scaling regime, shifting the dominant dependence from the full horizon $L$ to the longest sub-critical segment $h_{\max}$.
  Increasing landmark omission probability (not shown) increases $h_{\max}$ and restores the high-exponent scaling regime, confirming that structure functions primarily by constraining the uninterrupted autoregressive span rather than by providing additional semantic information.
  }
  \label{fig:trackB_scaling}
\end{figure}

The experimental results reveal three consistent trends.

First, under the unstructured condition, sample complexity exhibits exponential growth as the horizon length $L$ increases. The empirical slope closely matches the theoretical lower bound proportional to $|\mathcal{A}|^{L}$, providing empirical support for the analysis of TW-HSF-type tasks presented in Appendix~A.

Second, structural segmentation fundamentally alters the scaling regime. By partitioning the reasoning trajectory, the dominant exponential dependence shifts from the global horizon $L$ to the longest effective subsegment length $h_{\max}$. Rather than reducing constant factors, structure modifies the governing exponential scale itself, effectively localizing the curse of horizon to sub-critical segments.

Third, sensitivity to the landmark omission probability $p_{\mathrm{drop}}$ demonstrates that the primary role of structural signals is the truncation of autoregressive dependency. As landmarks fail and segments merge, the effective reasoning span increases, and the sample complexity rapidly returns toward the unstructured exponential regime.

From the perspective of the diagnostic framework introduced in Section~4.3, the exponential degradation observed in the unstructured condition corresponds to a dynamical regime dominated by sustained decay of decision advantage, characterized by elevated $\gamma_t$. Structural segmentation enforces resets that interrupt the exponential decay regime predicted by Theorem~A, preventing the reasoning trajectory from approaching the critical scale $L^*$ where instability becomes dominant.

\subsection{Chain Sensitivity Analysis: Process Stability in a Branch-Free Environment}
\label{sec:chain_sensitivity}

\subsubsection{Motivation and Setup}

To rigorously distinguish between failures arising from \emph{problem complexity} and those induced by \emph{process instability}, we construct a fully branch-free one-dimensional task environment. The state space is defined as $\{0,1,\dots,L\}$, where the sole objective is to advance from state $0$ to $L$. The branching factor is fixed at $b=1$, thereby eliminating search space explosion and isolating the intrinsic dynamics of autoregressive execution, consistent with classical planning formulations \cite{kaelbling1998planning,sutton1999between}.

Within this environment, we introduce two controlled sources of stochastic disturbance:

\begin{enumerate}
\item \textbf{Policy Noise:} with probability $\epsilon$, the intended forward action is reversed into a backward step.
\item \textbf{Sticky Trap:} with probability $p$, the previous executed action is inverted, simulating historical inertia and local dynamical trapping.
\end{enumerate}

A structured agent is additionally implemented, performing periodic phase resets every $K$ steps to prevent long-term accumulation of backward drift. Such resets are conceptually related to structured reasoning frameworks that introduce intermediate checkpoints or branching control \cite{yao2023tree,besta2024graph}.

This setup allows us to examine whether long-horizon collapse emerges even when traditional sources of complexity—such as combinatorial branching or semantic ambiguity—are entirely absent.

\subsection{Track A: Real LLM Behavior in TextWorld Environments}
\label{sec:trackA}

\subsubsection{Experimental Design}

To evaluate whether the structural constraints derived from Theorem~A manifest in real interactive environments, we conduct experiments in the TextWorld navigation framework \cite{cote2018textworld}. Unlike synthetic settings, this environment introduces realistic challenges including partial observability, historical dependency, and interactive dynamics.

We evaluate two language models: Gemma~3~4B and GPT-4o. To ensure rigorous and fair comparison, we employ a \textbf{paired-and-cached} experimental design. Within each trial, both the baseline agent and the structurally segmented (Landmarks) agent share identical language model outputs for the same $(\text{room}, \text{action set})$ input. Specifically:

\begin{itemize}
\item Model outputs are cached and reused across both agents;
\item Prompts, environment states, and available actions are strictly identical;
\item Behavioral differences arise solely from structural execution mechanisms.
\end{itemize}

The Landmarks agent introduces two structural governance mechanisms:

\begin{enumerate}
\item \textbf{Phase-boundary reset:} periodic segmentation of reasoning trajectories;
\item \textbf{In-phase edge deduplication:} pruning of redundant transitions within a segment.
\end{enumerate}

No additional semantic information, reward shaping, or model retraining is applied. Consequently, any observed behavioral divergence isolates the effect of reasoning topology rather than differences in model capability.

\subsubsection{Results and Analysis}

We present the results in increasing levels of empirical grounding: first, a trajectory-level visualization illustrating qualitative dynamics; second, quantitative measurements for Gemma~3~4B; and third, replication under GPT-4o.


\begin{figure}[t]
\centering
\begin{tikzpicture}
\begin{axis}[
    width=\columnwidth,
    height=6cm,
    xlabel={Interaction steps},
    ylabel={Distinct rooms visited},
    legend pos=north west,
    grid=major,
]

\addplot[mark=o] coordinates {
(0,0)
(10,6)
(20,10)
(30,12)
(40,13)
(50,13.5)
(60,13.8)
(70,14)
(80,14)
(90,14)
};

\addplot[mark=square*] coordinates {
(0,0)
(10,8)
(20,15)
(30,22)
(40,27)
(50,30)
(60,33)
(70,35)
(80,37)
(90,39)
};

\legend{Baseline, Landmarks}
\end{axis}
\end{tikzpicture}
\caption{
\textbf{Structural trajectory comparison in Track A.}
Baseline execution exhibits early saturation due to oscillatory attractors, while structured execution maintains sustained exploration growth.
}
\label{fig:trackA_trend}
\end{figure}
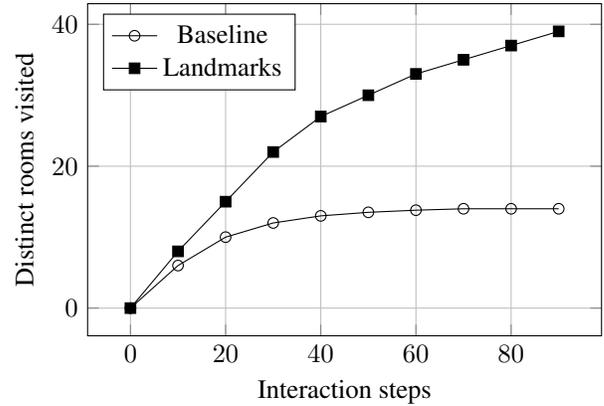

As illustrated in Figure~\ref{fig:trackA_trend}, structural segmentation alters the trajectory-level dynamics of exploration. The baseline execution rapidly saturates, indicating repeated revisitation of a limited subset of states, whereas the landmark-structured agent maintains continued exploratory growth. This qualitative pattern suggests that execution topology directly influences long-horizon stability.


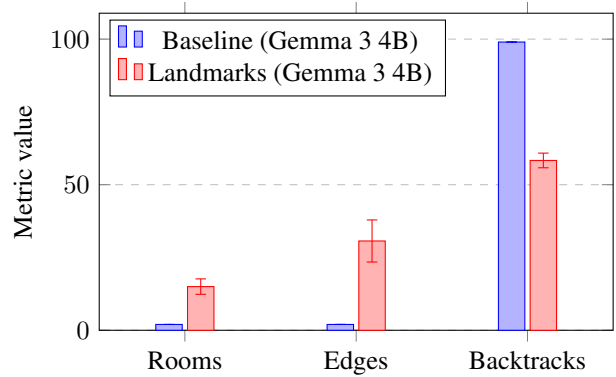
\begin{figure}[t]
\centering
\begin{tikzpicture}
\begin{axis}[
    width=\columnwidth,
    height=5.8cm,
    ybar,
    bar width=10pt,
    ymin=0,
    ylabel={Metric value},
    symbolic x coords={Rooms,Edges,Backtracks},
    xtick=data,
    legend style={at={(0.02,0.98)},anchor=north west},
    enlarge x limits=0.25,
    ymajorgrids=true,
    grid style=dashed,
    error bars/y dir=both,
    error bars/y explicit,
]

\addplot+[] coordinates {
(Rooms,2.00) +- (0,0.00)
(Edges,2.00) +- (0,0.00)
(Backtracks,99.00) +- (0,0.00)
};

\addplot+[] coordinates {
(Rooms,15.00) +- (0,2.65)
(Edges,30.67) +- (0,7.23)
(Backtracks,58.33) +- (0,2.52)
};

\legend{Baseline (Gemma~3~4B), Landmarks (Gemma~3~4B)}
\end{axis}
\end{tikzpicture}
\caption{
\textbf{Track A (Gemma~3~4B): Structural governance under paired-and-cached evaluation.}
Bars show mean $\pm$ std over paired trials. Structural resets increase exploration coverage and reduce oscillatory backtracking.
}
\label{fig:trackA_gemma_bars}
\end{figure}

The Gemma~3~4B results demonstrate that the baseline agent frequently becomes trapped in local oscillatory attractors, resulting in minimal exploration coverage. In contrast, structural segmentation substantially increases the number of explored states while reducing redundant backtracking behavior, suggesting that instability primarily emerges from execution dynamics rather than differences in model capability.


\begin{figure}[t]
\centering
\begin{tikzpicture}
\begin{axis}[
    width=\columnwidth,
    height=5.8cm,
    ybar,
    bar width=10pt,
    ymin=0,
    ylabel={Metric value},
    symbolic x coords={Rooms,Edges,Backtracks},
    xtick=data,
    legend style={at={(0.02,0.98)},anchor=north west},
    enlarge x limits=0.25,
    ymajorgrids=true,
    grid style=dashed,
    error bars/y dir=both,
    error bars/y explicit,
]

\addplot+[] coordinates {
(Rooms,5.00) +- (0,0.00)
(Edges,6.33) +- (0,0.58)
(Backtracks,96.00) +- (0,0.00)
};

\addplot+[] coordinates {
(Rooms,10.00) +- (0,0.00)
(Edges,18.67) +- (0,1.15)
(Backtracks,55.67) +- (0,13.32)
};

\legend{Baseline (GPT-4o), Landmarks (GPT-4o)}
\end{axis}
\end{tikzpicture}
\caption{
\textbf{Track A (GPT-4o): Structural governance under paired-and-cached evaluation.}
Structural resets improve exploration coverage and suppress oscillatory behavior.
}
\label{fig:trackA_4o_bars}
\end{figure}
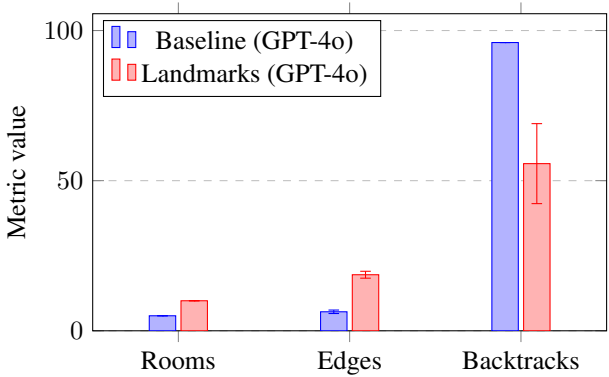

The GPT-4o experiments replicate the same qualitative pattern, indicating that the observed stability gains generalize across distinct model architectures. From the diagnostic perspective introduced in Section~4.3, the baseline condition effectively extends a single autoregressive trajectory, causing the effective edge-length ratio $r_t$ to approach the instability regime predicted by Theorem~A. Structural resets constrain trajectory length, maintaining execution within a more stable dynamical region.

Overall, these findings provide empirical evidence that topology-aware structural governance functions as a structurally grounded mechanism for sustaining coherent long-horizon reasoning in real interactive environments.

\subsection{Discussion: Search-based Compensation vs. Structural Governance}
\label{sec:discussion_branching}

\paragraph{On the Interaction of Branching and Decay.}
While our analysis primarily focuses on linear trajectories in order to isolate intrinsic process instability, introducing a branching factor $b > 1$ (e.g., via Beam Search or Tree-of-Thought \cite{yao2023tree}) induces a stochastic survival dynamic among multiple candidate trajectories. 

Under an independence approximation, let $\rho_t$ denote the per-step decision advantage for a single trajectory. The effective survival probability across $b$ parallel branches can be approximated as:
\begin{equation}
\hat{\rho}_t \approx 1 - (1 - \rho_t)^b,
\end{equation}
suggesting that branching can delay the onset of observable failure by increasing the likelihood that at least one trajectory maintains sufficient directional signal. 

Importantly, this mechanism does not alter the intrinsic decay constant $\gamma$ derived in Theorem~A; instead, it \textit{compensates} for decay through parallel exploration. As established in Appendix~A, however, branching introduces an exponential sample-complexity penalty ($|\mathcal{A}|^L$) with respect to horizon length. Consequently, while branching may extend the \emph{effective operational horizon}, it does so at substantial computational cost. 

By contrast, structural governance via segmentation and reset mechanisms avoids this exponential search overhead by constraining the uninterrupted autoregressive span. Rather than modifying the intrinsic decay dynamics, segmentation prevents execution from entering the instability regime associated with critical effective edge-length ratios ($r_t \to 1$). This distinction highlights a \textbf{fundamental trade-off}: search-based methods trade compute for local stability, whereas topology-based methods ensure global coherence through structural design.

\subsection{Summary}
\label{sec:chapter5_summary}

Across three complementary experimental tracks—synthetic extremes (Track~B), mechanistic isolation (Chain Sensitivity Analysis), and real interactive environments (Track~A)—the results collectively confirm the central theoretical prediction of this work: when reasoning is forced to accumulate along a single uninterrupted autoregressive trajectory, execution inevitably enters an instability regime. Structural segmentation and reset mechanisms therefore emerge not as heuristic enhancements, but as \emph{fundamental prerequisites} for preventing process-level collapse.

Taken together, the observed behavioral patterns align closely with the diagnostic framework introduced in Section~4.3. Rather than representing isolated empirical artifacts, long-horizon failure can be interpreted as a structural phase transition into a dynamical regime characterized by exponential decision-advantage decay, entropy escalation, and critical effective edge-length ratios ($r_t \to 1$). Within this framework, structural segmentation operates by constraining trajectory length below the critical scale $L^*$, thereby preempting entry into the unstable regime predicted by Theorem~A.

These findings provide cross-level empirical validation linking abstract theoretical analysis with observable execution dynamics. Crucially, structural governance simultaneously mitigates classical \emph{problem complexity}—traditionally addressed through search or scaling—and suppresses the \emph{intrinsic process instability} arising from autoregressive dynamics.

From a structural perspective, the experiments indicate that reasoning stability is governed less by semantic capability or prompt engineering than by the \textbf{execution topology} itself. This insight motivates the development of endogenous reset primitives and structured execution frameworks (e.g., SPR-like mechanisms), and suggests a broader shift in long-horizon reasoning research: from optimizing model scale toward governing the structural organization of inference.

\section{Conclusion and Future Work}
\label{sec:conclusion}

This work reframes reasoning from a scaling problem to a stability-governed dynamical system. We establish that failures in long-horizon reasoning originate not merely from task-level complexity, but from \emph{intrinsic dynamical instability} inherent to the autoregressive execution process itself.

From a theoretical perspective, Theorem~A establishes that the decision advantage along a single-path trajectory inevitably decays, defining a \textbf{fundamental stability horizon} $L^*$. Complementarily, Appendix~A demonstrates that without structural grounding, long-horizon tasks exhibit exponential learning complexity in the worst case. Together, these results suggest that stability constraints arise from the interaction between information-theoretic degradation and sequential execution dynamics.

Our analysis further indicates that DAG-like reasoning topologies should not be viewed as arbitrary architectural preferences, but as canonical organizational forms that naturally emerge under intrinsic stability constraints. This perspective provides a structural reinterpretation of recent advances in structured reasoning methods, including chain-of-thought prompting \cite{wei2022chain}, tree-based reasoning \cite{yao2023tree}, and graph-based reasoning frameworks \cite{besta2024graph}, positioning them as implicit responses to stability limitations rather than purely heuristic design choices.

Empirically, our three-track evaluation validates these predictions across multiple levels of abstraction. Track~B demonstrates the reconfiguration of exponential bottlenecks through structural segmentation; the mechanistic chain-sensitivity study isolates process-level collapse independent of search complexity; and Track~A reveals that structural governance significantly enhances exploration fidelity in real interactive environments \cite{cote2018textworld}, even when underlying language-model outputs are strictly controlled.

Taken together, these findings suggest a necessary conceptual shift in the development of intelligent systems: from monolithic linear scaling toward \textbf{structural governance}. While scaling laws successfully characterize aggregate performance trends \cite{kaplan2020scaling}, they remain largely silent on the dynamical stability properties of extended reasoning trajectories. Our results indicate that sustaining long-horizon reasoning may depend less on increasing model capacity alone and more on organizing inference through bounded segments ($\ell_i < L^*$) interconnected by stabilization mechanisms.

\begin{figure}[t]
\centering
\begin{tikzpicture}
\begin{axis}[
    width=0.9\linewidth,
    height=6cm,
    xlabel={Branching factor $b$},
    ylabel={Effective segment length $\ell$},
    xmin=1, xmax=10,
    ymin=0, ymax=60,
    grid=major,
    legend pos=north west,
]

\addplot[
    domain=1:10,
    samples=100,
    thick,
] {20 + 8*ln(x)};  

\addlegendentry{Phase boundary}

\addplot[
    name path=A,
    domain=1:10,
] {0};

\addplot[
    name path=B,
    domain=1:10,
] {20 + 8*ln(x)};

\addplot[
    fill opacity=0.1
] fill between[of=A and B];

\node at (axis cs:3,10) {Stable regime};

\node at (axis cs:6,45) {Instability regime};

\node at (axis cs:8,25) {Branching-compensated};

\end{axis}
\end{tikzpicture}

\caption{
Conceptual phase diagram of long-horizon autoregressive reasoning.
The phase boundary illustrates the trade-off between branching factor $b$
and effective segment length $\ell$. While increased branching delays
failure by improving trajectory survival probability, intrinsic decay
imposes a logarithmic limit. Structural segmentation constrains $\ell$
below the instability boundary without incurring exponential search cost.
}
\label{fig:phase_diagram}
\end{figure}
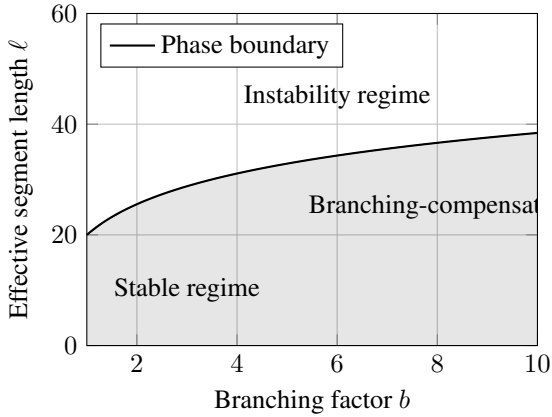

Several promising directions remain for future investigation:

\begin{enumerate}
\item \textbf{Autonomous Node Optimization.} 
Developing adaptive mechanisms for node placement, potentially guided by entropy-based diagnostics such as $\gamma_t$ or $\Delta H_t$, to enable endogenous stabilization without external supervision.

\item \textbf{Framework Generalization.}
Extending the stability framework to multi-agent collaboration, tool-augmented reasoning, or external memory systems, examining whether structural constraints persist across broader execution paradigms.

\item \textbf{Diagnostic Operationalization.}
Empirically validating the proposed dynamical indicators as real-time monitoring tools for reasoning reliability in large-scale LLM benchmarks.
\end{enumerate}

Finally, this work invites a reconsideration of how reasoning systems are evaluated. If $L^*$ represents an intrinsic stability scale, traditional short-horizon evaluation paradigms may suffer from a form of sub-critical bias, potentially overlooking structural instabilities that only emerge beyond limited interaction depths \cite{chollet2019measure}. Models may demonstrate high local coherence while remaining structurally unstable over extended reasoning horizons. From this perspective, future evaluation frameworks may need to move beyond short-term behavioral imitation toward analyzing how information is structured, compressed, and preserved across long temporal trajectories.

In summary, the challenge of long-horizon reasoning may not primarily be a question of scaling capability, but one of structural organization. By reframing long-range reasoning as a problem of structural governance, this work provides a principled foundation for the design and analysis of next-generation reasoning architectures.

\bibliographystyle{unsrt} 
\bibliography{references} 

\appendix
\section{Appendix A: Learning-Theoretic Lower Bounds in Sparse Long-Horizon Text Environments (A Problem-Complexity Baseline)}
\label{app:lower_bounds}

\subsection{Appendix Overview}
This appendix analyzes a class of sparse long-horizon text environments that induce exponential learning difficulty due to partial observability, aliasing, and reward sparsity. The results here establish problem-complexity lower bounds of the classical reinforcement learning and POMDP type.

\paragraph{Scope clarification.} The lower bounds derived in this appendix concern learning and exploration complexity \textit{across episodes}. They do not address the stability of \textit{single-trajectory} inference execution, which is the focus of Theorem~A in the main text. The appendix serves as a contrastive baseline, highlighting that long-horizon failure may arise even when inference execution is assumed ideal.

\subsection{Problem Setup: The TextWorld Hard Sparse Family ($\text{TW-HSF}_\varepsilon$)}
\label{sec:tw_hsf_setup}

We define a worst-case family of sparse long-horizon POMDPs that isolates aliasing, sparsity, and procedural dependency.

\begin{definition}[TextWorld Hard Sparse Family, $\text{TW-HSF}_\varepsilon$]
Let $\mathcal{M}(L, \mathcal{A}, \varepsilon)$ be a family of POMDPs parameterized by horizon $L \in \mathbb{N}$, a finite action space $\mathcal{A}$ with $|\mathcal{A}| \ge 2$, and an aliasing factor $\varepsilon \in [0,1)$. Each instance $M \in \mathcal{M}$ has:
\begin{itemize}
    \item \textbf{Hidden Unique Optimal Path:} A unique action sequence $a^*_{1:L} \in \mathcal{A}^L$ that deterministically transitions through latent optimal states $s_1^* \to \dots \to s_L^*$. Any deviation from $a_t^*$ transitions to an absorbing failure state $s_{\mathsf{fail}}$.
    \item \textbf{Sparse Terminal Reward:} $r(s_t, a_t) = 1$ iff $t=L$ and $a_{1:L} = a^*_{1:L}$; otherwise $r = 0$.
    \item \textbf{$\varepsilon$-Aliased Observations:} There exists a fixed distribution $D$ over observation space $\Omega$ such that for all non-terminal optimal states $s_t^*$ with $t < L$, 
    \[ \mathrm{TV}\bigl(\Omega(\cdot \mid s_t^*), D\bigr) \le \varepsilon. \]
    \item \textbf{Absence of Structural Signals:} Observations encode no explicit time index, progress marker, admissible action constraints, or subgoal completion signals.
\end{itemize}
The optimal path $a^*_{1:L}$ is sampled uniformly from $\mathcal{A}^L$.
\end{definition}

\subsection{Indistinguishability of Histories}
Aliasing prevents reliable depth inference from observation histories. Let $h_t = (o_{1:t}, a_{1:t-1})$ denote the history at time $t$.

\begin{lemma}[History Indistinguishability]
Fix any policy $\pi$. Condition on survival (i.e., no transition to $s_{\mathsf{fail}}$) up to the corresponding time. For any $t, t' < L$:
\[ \mathrm{TV}\Bigl(\mathbb{P}_\pi(h_t \mid \mathrm{survive}(t)), \mathbb{P}_\pi(h_{t'} \mid \mathrm{survive}(t'))\Bigr) \le \varepsilon \cdot \min(t, t'). \]
\end{lemma}

\begin{proof}
Conditioned on survival, action prefixes are consistent with the optimal path. Differences between histories arise solely from observations. For the first $k = \min(t, t')$ steps, each observation distribution is within $\varepsilon$ total variation distance of the stationary distribution $D$. By tensorization and triangle inequality:
\[ \mathrm{TV}\bigl((o_1, \dots, o_k), (o'_1, \dots, o'_k)\bigr) \le \sum_{i=1}^k \mathrm{TV}(\Omega(\cdot \mid s_i^*), D) \le k\varepsilon. \]
Conditioning preserves total variation bounds by data processing.
\end{proof}

\begin{corollary}[Depth Unidentifiability]
If $\varepsilon \le c/L$ for a sufficiently small constant $c > 0$, then for any $t, t' < L$:
\[ \mathrm{TV}\bigl(\mathbb{P}_\pi(h_t \mid \mathrm{survive}), \mathbb{P}_\pi(h_{t'} \mid \mathrm{survive})\bigr) = O(c) = o(1). \]
In particular, no agent can reliably infer its depth $t$ from its history under survival.
\end{corollary}

\subsection{Single-Episode Success Probability}

\begin{lemma}[Near-Random Action Selection]
For any policy $\pi$ and any $t < L$:
\[ \mathbb{P}_\pi(a_t = a_t^* \mid h_t, \mathrm{survive}(t)) \le \frac{1}{|\mathcal{A}|} + O(t\varepsilon). \]
\end{lemma}

\begin{lemma}[Single-Episode Success Probability]
For any policy $\pi$, if $\varepsilon \le c/L$:
\[ \mathbb{P}_\pi(\mathrm{success}) \le C \cdot |\mathcal{A}|^{-L}, \]
for constants $c, C > 0$.
\end{lemma}

\subsection{Exponential Episode Complexity}

\begin{theorem}[Exponential Episode Complexity]
To achieve success probability at least $\delta \in (0,1)$ on $\text{TW-HSF}_\varepsilon$ using $N$ independent episodes:
\[ N \ge \Omega\left(|\mathcal{A}|^L \log \frac{1}{\delta}\right), \]
for $\varepsilon \le c/L$.
\end{theorem}

\subsection{Structural Augmentation via Landmarks}
We formalize the effect of learning-time structure by partitioning the horizon into subtasks separated by identifiable landmarks.

\begin{corollary}[Reduction via Landmarks]
Assume there exist $k \ge 2$ identifiable landmarks partitioning the task into subtasks of lengths $\ell_1, \dots, \ell_k$, with $\sum_i \ell_i = L$. Let $\ell_{\max} = \max_i \ell_i$. If an agent can detect landmarks and learn each subtask independently, then:
\[ N_{\mathrm{structured}} \ge \Omega\left(k \cdot |\mathcal{A}|^{\ell_{\max}} \log \frac{1}{\delta}\right). \]
In the near-uniform case ($\ell_i \approx L/k$):
\[ N_{\mathrm{structured}} \ge \Omega\left(k \cdot |\mathcal{A}|^{L/k}\right). \]
\end{corollary}

\subsection{Interpretation and Relation to the Main Text}
The results in this appendix demonstrate that end-to-end learning without structure is \textbf{exponentially intractable} in sparse long-horizon text environments, even when inference execution is assumed ideal.

These bounds, however, address learning complexity across episodes and do not explain failures that arise within a single execution trajectory. In contrast, \textbf{Theorem A} in the main text characterizes a distinct limitation: the instability of autoregressive inference execution itself, which persists even in settings where branching and search complexity vanish.

Together, these results indicate that long-horizon reasoning may fail for two orthogonal reasons:
\begin{enumerate}
    \item \textbf{Problem complexity} (as captured here in Appendix A), and
    \item \textbf{Process instability} (as characterized by Theorem A in the main text).
\end{enumerate}

\section{Appendix B: Proof Outline for Theorem A and Testable Contraction Assumptions}
\label{app:theorem_a}

This appendix provides a proof outline for Theorem~A and clarifies the minimal assumption needed to obtain exponential decay of decision advantage under single-path autoregressive execution. 
Our main message is that exponential decay follows from a \emph{process-level contraction property} of the induced decoding kernel. 
We do \emph{not} claim a universal closed-form derivation of the contraction constant from architectural primitives (e.g., softmax) in this version; instead, we state sufficient and testable conditions under which such contraction holds.

\subsection{Restatement of Theorem A (Execution Stability Horizon)}
Consider an autoregressive reasoning trajectory as a Markov process
\[
Z_{t+1} \sim K(\cdot \mid Z_t),
\]
where $Z_t$ is the latent execution state at step $t$ and $K$ is the induced transition kernel of decoding and internal updates.
Let $G \in \{0,1\}$ denote the target conclusion (correct vs.\ incorrect). We define the decision advantage at step $t$ as the Bayes-optimal testing advantage:
\[
\rho_t := 1 - 2P_e(t),
\quad
P_e(t) := \inf_{\phi}\Pr[\phi(Z_t)\neq G],
\]
where the infimum is over all measurable predictors $\phi$. Hence $\rho_t \in [0,1]$ and $\rho_t=0$ corresponds to chance-level inference.

\textbf{Theorem A (informal).}
If the induced kernel loses a fixed fraction of goal-relevant information per step (formalized below), then $\rho_t$ decays exponentially with $t$, implying a finite stability horizon $L^*$ beyond which reliable single-path execution becomes unstable.

\subsection{Objects and Notation}
Let
\[
P_t := \mathcal{L}(Z_t \mid G=1), 
\qquad 
Q_t := \mathcal{L}(Z_t \mid G=0),
\]
denote the conditional state distributions. We use total variation (TV) distance
\[
\Delta_t := \TV(P_t,Q_t)
\]
as a distinguishability measure, and mutual information $I(G;Z_t)$ as a goal-relevant information measure.

We assume balanced priors $\Pr(G=1)=\Pr(G=0)=1/2$ for simplicity; this can be relaxed with minor constant changes.

\subsection{Main Assumption: Per-step Contraction of the Induced Decoding Channel}
The proof of Theorem~A relies on a single structural property of the induced kernel.

\begin{assumption}[Kernel contraction (TV form)]
\label{ass:tv_contraction}
There exists $\eta \in (0,1)$ such that for any pair of input distributions $\mu,\nu$ over $Z_t$,
\[
\TV(\mu K, \nu K)\ \le\ \eta \,\TV(\mu,\nu),
\]
where $\mu K$ denotes the pushforward distribution after one step of the Markov kernel $K$.
\end{assumption}

\noindent
\textbf{Interpretation.}
Assumption~\ref{ass:tv_contraction} isolates \emph{intrinsic process-level instability} as a property of the induced decoding channel itself, independent of branching, search complexity, or task semantics. It says that one step of autoregressive execution makes two hypotheses about the goal strictly less distinguishable.

\paragraph{Testability.}
The constant $\eta$ is an \emph{effective contraction coefficient} that can be estimated empirically by fitting distinguishability decay across controlled tasks (see Section~\ref{sec:operational_alpha}).

\paragraph{Sufficient conditions (informal).}
Assumption~\ref{ass:tv_contraction} holds for many ergodic/noisy Markov kernels with a Doeblin/minorization condition or a Dobrushin contraction coefficient strictly less than $1$. This includes kernels that (i) apply a deterministic map followed by (ii) an irreducible noise channel that is independent of $G$ and has a uniform mixing component. Standard references on contraction coefficients and strong data processing inequalities apply (e.g., lecture notes and surveys such as \cite{polyanskiy2014lecture}).

\subsection{From Contraction to Exponential Decay}
\begin{lemma}[Iterated contraction implies exponential decay in TV]
\label{lem:tv_iter}
Under Assumption~\ref{ass:tv_contraction}, we have
\[
\Delta_t \le \eta^t \Delta_0.
\]
\end{lemma}

\begin{proof}
By definition, $P_{t+1}=P_t K$ and $Q_{t+1}=Q_t K$. Applying Assumption~\ref{ass:tv_contraction} with $\mu=P_t$ and $\nu=Q_t$ yields
\[
\Delta_{t+1} = \TV(P_tK,Q_tK) \le \eta\,\TV(P_t,Q_t)=\eta\,\Delta_t.
\]
Iterating gives $\Delta_t \le \eta^t\Delta_0$.
\end{proof}

\begin{lemma}[Decision advantage is controlled by distinguishability]
\label{lem:adv_tv}
With balanced priors, the Bayes-optimal testing advantage satisfies
\[
\rho_t \le \Delta_t.
\]
\end{lemma}

\begin{proof}
For testing $G\in\{0,1\}$ from $Z_t$ under equal priors, the Bayes error satisfies
\[
P_e(t) = \tfrac12\bigl(1-\TV(P_t,Q_t)\bigr),
\]
and hence $\rho_t = 1-2P_e(t) = \TV(P_t,Q_t)=\Delta_t$. 
(If one uses an alternative definition of $\rho_t$ (e.g., via a fixed decision rule), the inequality form $\rho_t\le \Delta_t$ still holds.)
\end{proof}

\begin{proof}[Proof sketch of Theorem~A]
Combine Lemma~\ref{lem:tv_iter} and Lemma~\ref{lem:adv_tv}:
\[
\rho_t \le \Delta_t \le \eta^t\Delta_0.
\]
Let $\rho_0 := \Delta_0$ and $\gamma := -\ln \eta>0$. Then
\[
\rho_t \le \rho_0 e^{-\gamma t}.
\]
Solving $\rho_t \le \tau$ yields the critical horizon
\[
L^* = \frac{1}{\gamma}\ln\frac{\rho_0}{\tau}.
\]
\end{proof}

\subsection{Equivalent Formulation via SDPI (Mutual Information Form)}
Some readers may prefer an information-theoretic statement that cleanly separates (i) per-step information loss from (ii) hypothesis-testing conversion.

\begin{assumption}[Per-step SDPI (MI form)]
\label{ass:sdpi}
There exists $\alpha\in(0,1)$ such that
\begin{equation}
\begin{aligned}
I(G;Z_{t+1}) &\le \alpha\, I(G;Z_t) \\
&\text{for all joint laws of } (G,Z_t)
\end{aligned}
\end{equation}
consistent with the process.
\end{assumption}

\noindent
Assumption~\ref{ass:sdpi} is a strong data processing inequality (SDPI) for the induced channel $Z_t\mapsto Z_{t+1}$, commonly satisfied by mixing/noisy channels under suitable regularity \cite{polyanskiy2014lecture}.

\begin{lemma}[SDPI implies exponential decay of MI]
\label{lem:mi_decay}
Under Assumption~\ref{ass:sdpi}, $I(G;Z_t)\le \alpha^t I(G;Z_0)$.
\end{lemma}
\begin{proof}
Iterate Assumption~\ref{ass:sdpi}.
\end{proof}

\begin{lemma}[MI controls testing advantage (up to constants)]
\label{lem:mi_to_adv}
There exists a universal constant $c>0$ such that
\[
\rho_t \le c\sqrt{I(G;Z_t)}.
\]
\end{lemma}

\begin{proof}[Proof sketch]
This follows from standard binary hypothesis-testing inequalities (e.g., Le Cam/Bretagnolle--Huber type bounds). See \cite{lecam1986asymptotic} for classical references. Constants are not central to our claims.
\end{proof}

Combining Lemma~\ref{lem:mi_decay} and Lemma~\ref{lem:mi_to_adv} yields
\[
\rho_t \le c\sqrt{I(G;Z_0)}\,\alpha^{t/2}
= \rho_0 e^{-\gamma t},
\quad
\gamma := -\tfrac12\ln\alpha>0.
\]

\subsection{Operational Definition: Estimating $\alpha(C)$ or $\eta(C)$ from Controlled Tasks}
\label{sec:operational_alpha}
To make the contraction assumption reviewer-auditable, we propose an operational estimate.

Fix a controlled task family indexed by a complexity parameter $C$ (e.g., synthetic legal reasoning instances with controlled distractor load, depth, or aliasing). For each $t$, measure the Bayes error (or a calibrated proxy) of predicting $G$ from $Z_t$, yielding $P_e(t;C)$ and advantage
\[
\rho_t(C) = 1 - 2P_e(t;C).
\]
Fit an exponential curve over the range where decay is approximately geometric:
\[
\rho_t(C) \approx \rho_0(C)\,e^{-\gamma(C)t}.
\]
Then define the effective coefficients
\[
\eta(C) := e^{-\gamma(C)}, 
\qquad
\alpha(C) := e^{-2\gamma(C)}.
\]
This turns the abstract stability assumption into an empirically estimable quantity.

\subsection{Toy Numerical Illustration (Not a Claim about LLMs)}
To illustrate the mechanism in a minimal setting, consider the 1D autoregressive chain
\[
Z_{t+1} = a Z_t + \epsilon_t, \qquad \epsilon_t\sim\mathcal{N}(0,\sigma^2),
\]
with $|a|<1$. Two initial hypotheses about $Z_0$ become less distinguishable over time, producing geometric decay in any reasonable distinguishability proxy.

\begin{verbatim}
import numpy as np
np.random.seed(42)
T = 100
eta_est = 0.95
delta = [0.5]
for t in range(1, T):
    delta.append(eta_est * delta[-1])
# delta[t] decays geometrically; gamma = -ln(eta_est)
\end{verbatim}

This toy example is included solely to visualize geometric decay; it is not intended as a faithful mechanistic model of Transformer internals.

\subsection{Remarks on Scope}
\begin{remark}[What this appendix does and does not claim]
Theorem~A is a structural result: once the induced decoding dynamics admit a strict per-step contraction (TV) or SDPI (MI), exponential decay of decision advantage follows. Deriving an exact contraction constant from specific architectural primitives (e.g., attention softmax) is a separate, more model-specific question and is left for future work.
\end{remark}

\end{document}